\documentclass[conference]{IEEEtran}
\IEEEoverridecommandlockouts

\usepackage{amsmath,amssymb,amsfonts}
\usepackage{algorithmic}
\usepackage{graphicx}
\usepackage{textcomp}
\usepackage{xcolor}
\def\BibTeX{{\rm B\kern-.05em{\sc i\kern-.025em b}\kern-.08em
    T\kern-.1667em\lower.7ex\hbox{E}\kern-.125emX}}
    
\usepackage{cite}  % use natbib instead
\usepackage{cleveref}
\usepackage{comment}
\usepackage{algorithm}
\usepackage{booktabs}
\usepackage{enumitem}

\usepackage{times}
\usepackage{courier}
\usepackage{helvet}
\usepackage[hyphens]{url}
\usepackage{caption}
\usepackage{subcaption}
\usepackage{mathtools}
\usepackage{accents}

\usepackage{stackengine}

\usepackage{amsthm}
\theoremstyle{plain}
\newtheorem{proposition}{Proposition}

\theoremstyle{definition}

\newtheorem{example}{Example}[section]

\theoremstyle{remark}
\newtheorem{remark}{Remark}[subsection]

\begin{document}
\title{
Learning Passive Continuous-Time Dynamics with Multistep Port-Hamiltonian Gaussian Processes
}

\author{Chi Ho Leung and Philip E. Par\'{e}*
    \thanks{*Chi Ho Leung and Philip E. Par\'e are with the Elmore Family School of Electrical and Computer Engineering, Purdue University, USA.
    E-mail: leung61@purdue.edu,philpare@purdue.edu. 
    This material is based upon work supported in part by the National Science Foundation (NSF-ECCS \#2238388).
    }
}

\maketitle

\begin{abstract}
We propose the multistep port-Hamiltonian Gaussian process (MS-PHS GP) to learn physically consistent continuous-time dynamics and a posterior over the Hamiltonian from noisy, irregularly-sampled trajectories. 
By placing a GP prior on the Hamiltonian surface $H$ and encoding variable-step multistep integrator constraints as finite linear functionals, MS-PHS GP enables closed-form conditioning of both the vector field and the Hamiltonian surface without latent states, while enforcing energy balance and passivity by design. 
We state a finite-sample vector-field bound that separates the estimation and variable-step discretization terms.
Lastly, we demonstrate improved vector-field recovery and well-calibrated Hamiltonian uncertainty on mass-spring, Van der Pol, and Duffing benchmarks.
\end{abstract}


\section{Introduction}
Learning physically consistent continuous-time dynamics is a core requirement for modeling and control, but noise and irregular sampling complicate identification.
Gaussian processes (GPs) provide a principled Bayesian framework for uncertainty quantification in system identification \cite{rasmussen2003gaussian, alvarez2012kernels}, yet standard GP approaches typically ignore important structure such as energy conservation and passivity that are intrinsic to many mechanical and electrical systems.
Recent work has begun to combine GP learning with port-Hamiltonian system (PHS) constraints, injecting physics-informed inductive biases via tailored kernels \cite{beckers2022gaussian, beckers2023data}. 
These approaches directly enforce the port-Hamiltonian \cite{van2014port} structure in the learned dynamical model.
However, these methods do not fully address the practical obstacles of measurement noise and irregular sampling in trajectories.

To close this gap, we introduce the \emph{multistep port-Hamiltonian system} (MS-PHS) Gaussian process kernel. 
Our construction places a GP prior on the Hamiltonian surface $H(x)$ and leverages variable-step linear multistep methods (vLMM) to express discrete, irregularly-sampled trajectory constraints as finite linear functionals of $H$ and $\nabla H$.
Since GPs are closed under linear operations, this yields closed-form posterior means and covariances for both the continuous vector field $f(x)=[J(x)-R(x)]\nabla H(x)$ and the Hamiltonian surface $H(x)$, directly from noisy, irregular data 
without leveraging a separate data preprocessing stage for extracting vector field information from trajectory data. 
By construction, the approach preserves PHS properties such as energy balance and passivity while accurately calibrating uncertainty from noisy, irregular trajectory observations.

We summarize a preliminary estimation-error decomposition that separates numerical truncation via the vLMM order and step size from GP interpolation uncertainty.
The formal theorem and proof appear in a companion preprint~\cite{leung2025spectral}, and are not required for the methods and experiments herein.
Lastly, we empirically evaluate MS-PHS on three canonical oscillators (mass-spring, Van der Pol, Duffing) under noisy, irregular sampling, comparing against MS-ODE~\cite{ensinger2024exact} and GP-PHS~\cite{beckers2022gaussian} (with Savitzky-Golay/LOESS filters).
We report that across vector‐field recovery and Hamiltonian‐posterior calibration metrics, MS-PHS matches or exceeds the baselines while delivering uncertainty that tracks the true error, especially as noise and timestamp jitter increase.

\subsection{Related Works}
The port-Hamiltonian system (PHS) formalism provides a unifying language for interconnection, energy balance, and passivity, with modular composition and stability guarantees~\cite{van2014port, maschke1993port}.
Building on this foundation, recent identification approaches learn models that retain passivity by design~\cite{beckers2022gaussian, zaspel2024data, neary2023compositional}. 
These advances clarify what physical invariants to encode when learning passive continuous-time systems, and naturally motivate probabilistic treatments under noisy, irregular sampling.

Gaussian processes offer nonparametric priors with closed-form posteriors under linear functionals and have been used effectively for continuous-time vector-field recovery~\cite{heinonen2018learning}.
In many pipelines, derivative preprocessing and regular sampling are pragmatic choices that work well~\cite{beckers2022gaussian, jagtap2020control, brunton2016discovering}. 
At the same time, when samples are irregular and end-to-end uncertainty quantification is essential, it is appealing to bring discretization inside the statistical model rather than treating it as a separate preprocessing stage.

The vLMM schemes such as AB/AM/BDF~\cite{hairer1993solving} translate flows into linear constraints in stacked vector field evaluations. 
Recent work shows that these constraints can be projected into GP posteriors for exact inference, even with variable step size~\cite{ensinger2024exact}. 
This line of work substantially advances GP-based identification under irregular sampling. 
A key contribution of our work is to integrate physics-informed PHS priors~\cite{beckers2022gaussian} into the vLMM-based GP~\cite{ensinger2024exact} to propagate noise and sampling irregularity into the Hamiltonian posterior. Our proposed GP framework enables the quantification of geometric uncertainty with meaningful physical interpretations.

The paper is organized as follows: Section~\ref{sc:background} reviews background on PHS and GPs. Section~\ref{sc:method} develops the MS-PHS kernel and theoretical results. Section~\ref{sc:experiment} presents experiments.

\subsection{Notations}
The notation \(\mathbb{R}\) denotes the real number line. 
Vectors in \(\mathbb{R}^n\) are column vectors. 
The Kronecker product is denoted as $\otimes$.
The matrix $I_n$ denotes the $n\times n$ identity.
The operator ${\rm vec}$ denotes column-wise vectorization, mapping $X \in \mathbb{N}^{m\times n}$ into ${\rm vec}(X) \in \mathbb{R}^{mn}$.
$[A]_{ij}$ denotes the $i,j$ entry of a matrix $A$.
We write ${\rm det}(\cdot)$ for the determinant of a matrix.
Gradient of scalar function $H$ with respect to $x$ is denoted as $\nabla_x H$.
We write \(\mathbb E[\cdot]\) for expectation.
The GP distribution is denoted as $\mathcal{GP}(\cdot, \cdot)$.

\section{Background}\label{sc:background}
Necessary mathematical tools are introduced here.
\subsection{Port-Hamiltonian Systems}
Energy-conserving and dissipative systems with input/output ports can be formally described as port-Hamiltonian systems (PHS) \cite{van2014port}:
\begin{equation*}
\begin{split}
    \dot {{x}} &= [J({x}) - R({x})]\nabla_{{x}} H({x}) + G({x}){u}\\
    {y}_{\rm out}      &= G({x})^\top \nabla_{{x}} H({x})
\end{split}
\end{equation*}
in which state ${x}(t) \in \mathbb{R}^n$ and the I/O ports ${u}(t), {y}_{\rm out}(t) \in \mathbb{R}^m$ evolve accordingly with time $t \in \mathbb{R}_{\geq 0}$. 
The Hamiltonian $H: \mathbb{R}^n \to \mathbb{R}$ is a smooth function that represents the total energy stored in the system.
The interconnection matrix $J: \mathbb{R}^n \to \mathbb{R}^{n\times n}$ is skew-symmetric.
The dissipative matrix $R: \mathbb{R}^n \to \mathbb{R}^{n\times n}$ is a positive semi-definite such that $R({x})=R^\top({x}) \succeq 0$.
The port mapping matrix $G: \mathbb{R}^n \to \mathbb{R}^{n\times m}$ defines how the external input/output ports ${u}$, ${y}_{\rm out}$ are coupled to the energy storage dynamics.

\subsection{Gaussian Process Regression (GPR)}
A GP is a Bayesian prior over functions $f:\mathbb{R}^d \to \mathbb{R}$, where $f\sim \mathcal{GP}(m(\cdot), k(\cdot, \cdot))$, such that any finite collection of function values is jointly Gaussian \cite{rasmussen2003gaussian}.  
The GP distribution, $\mathcal{GP}(m(\cdot), k(\cdot, \cdot))$, is fully specified by the mean and kernel function:
\begin{equation}
  m({x}) = \mathbb{E}[f({x})],\quad  k({x},{x}') = \operatorname{Cov}[f({x}),f({x}')].
\end{equation}
Given noisy observations $\mathcal{D} = \{({x}_i, y_i)\}_{i=1}^N$ with
\begin{equation}
    y_i = f({x}_i) + \varepsilon_i,\quad \varepsilon_i\sim\mathcal{N}(0,\sigma_y^2),
\end{equation}
the posterior over function values at test inputs $X_* = [{x}_*^1,\dots,{x}_*^L]^\top$ is also Gaussian,
\begin{align}
  f_* \mid \mathcal{D}, X_* &\sim \mathcal{N}\bigl(\mu_*, \Sigma_*\bigr),
\end{align}
with posterior predictive mean and covariance:
\begin{equation}
\begin{split}
  \mu_* &= m(X_*) + K_{*X}\bigl(K_{XX} + \sigma_y^2 I\bigr)^{-1}({y} - m(X)),\\
  \Sigma_* &= K_{**} - K_{*X}\bigl(K_{XX} + \sigma_y^2 I\bigr)^{-1}K_{X*},
\end{split}
\end{equation}
where $X = [{x}_1,\dots,{x}_N]^\top$, ${y} = [y_1, \dots, y_N]^\top$, $[K_{XX}]_{ij} = k({x}_i,{x}_j)$, $[K_{*X}]_{ij} = k({x}_*^i,{x}_j)$, and $[K_{}]_{ij} = k({x}_*^i,{x}_*^j)$.

The noise variance $\sigma_y^2$ and kernel hyperparameters, e.g. length‑scales and output variance, are learned by minimizing the negative log marginal likelihood,
\begin{equation}\label{eq:nll}
\begin{split}
  -\log p({y}\mid X) &= \frac12 {y}^\top\bigl(K_{XX} + \sigma_y^2 I\bigr)^{-1}{y}\\
    &\qquad +\frac12\log\det\bigl(K_{XX} + \sigma_y^2 I\bigr)
    +\frac{N}{2}\log 2\pi.
\end{split}
\end{equation}
For vector‑valued functions, such as continuous‑time dynamics ${f}:\mathbb{R}^n\to\mathbb{R}^n$, one can place independent GP priors on each output dimension or adapt a matrix kernel as detailed in \cite{alvarez2012kernels}.

\subsubsection{Port–Hamiltonian System Kernel}
To enforce port-Hamiltonian dynamics on noisy vector field data $\{({x}_k, {f}({x}_k) + \varepsilon_k)\}_{k=1}^K$, where: 
\begin{equation*}
    {f}({x}_k) \coloneq [J({x}_k) - R({x}_k)]\nabla_{{x}_k} H({x}_k).
\end{equation*}
A zero‐mean GP prior is proposed in \cite{beckers2022gaussian} to place on the unknown Hamiltonian:
\begin{equation}\label{eq:Ham_prior}
    H({x})\sim\mathcal{GP}\bigl(0, k_{\rm base}({x},{x}')\bigr).
\end{equation}
Under the PHS dynamics \eqref{eq:PHS}, closure of GPs under affine transformations induces the matrix‐valued PHS kernel:
\begin{equation}\label{eq:phs-kernel}
  k_{\rm phs}({x},{x}')
    = \sigma_f^2 J_R({x}) \nabla_{{x}}\nabla_{{x}'} k_{\rm base}({x}, {x}') J_R({x}')^\top,
\end{equation}
where \(J_R({x})=J({x})-R({x})\) encodes the interconnection and dissipation matrices, and $\sigma_f^2 \in \mathbb{R}_{>0}$ encodes the signal noise.
The above setup yields a GP prior:
\begin{equation}\label{eq:phs_prior}
    \dot{{x}} \sim \mathcal{GP}(G({x}){u}, k_{\rm phs}({x},{x}'))
\end{equation}
for the PHS dynamics in \eqref{eq:PHS}.
\begin{remark}
    For concreteness, we use the squared exponential kernel for $k_{\rm base}$, while noting that other kernel choices are equally possible.
    Specifically,
    \begin{equation}
        k_{\rm base}({x}, {x}')
        = \exp\Bigl(-\frac12 \sum_{i=1}^n\frac{({x}_{i}-{x}_{i}')^2}{\ell_i^2}\Bigr),
    \end{equation}
    where $\ell = [\ell_1, \dots, \ell_n]^\top$ is the vector of length-scales, with $\ell_i \in \mathbb{R}_{\geq 0},\ \forall i = 1, \dots, n$.
    This corresponds to the automatic relevance determination (ARD) parameterization, which assigns an individual length-scale to each state dimension.
    
\end{remark}

\section{Method}\label{sc:method}
We formulate the problem and introduce our method here.
\subsection{Problem Formulation}
Let $\mathcal{D} = \{(\tilde{{x}}(t_k), {u}(t_k), t_k)\}_{k=1}^K$
denote the observations of the states and input signals of a PHS:
\begin{equation}\label{eq:PHS}
    \dot {{x}} = [J({x}) - R({x})]\nabla_{{x}} H({x}) + G({x}){u},
\end{equation}
sampled at time $t_1,\dots,t_K$, where:
\begin{equation}\label{eq:noisy_observations}
    \tilde{{x}}(t_k) = {x}(t_k) + \varepsilon_k, \quad \varepsilon_k \sim \mathcal{N}(0, \sigma_{{x}}^2I),
\end{equation}
with $\sigma_{{x}}^2\in \mathbb{R}_{>0}$ and ${x}(t)$ satisfying Equation~\eqref{eq:PHS} under initial condition ${x}(0) \in \mathbb{R}^n$.
We follow the setting in \cite{beckers2023data} concerning our prior knowledge of \eqref{eq:PHS}:
\begin{enumerate}
    \item the Hamiltonian $H \in \mathcal{C}^\infty$ is unknown;
    \item the parametric structures of $J(\cdot), R(\cdot), G(\cdot)$ are known;
    \item the parameters of $J(\cdot), R(\cdot), G(\cdot)$ are unknown.
\end{enumerate}
We further assume that the PHS dynamics can be approximated using a specified numerical integration scheme~\cite{atkinson2009numerical}.

% Given the irregular, noisy trajectory dataset $\mathcal{D}$ generated by the PHS in~\eqref{eq:PHS}–\eqref{eq:noisy_observations} with unknown Hamiltonian $H$ and unknown parameters in $(J,R,G)$ with known parametric forms, 
Under the above setup, our goal is to learn a probabilistic, physics-consistent model of the zero-input drift
$f(x) = [J(x)-R(x)]\nabla_x H(x)$ and of the Hamiltonian surface $H(x)$ that:
(i) combines vLMM constraints and PHS structure in a single GP likelihood;
(ii) derives closed-form posterior mean and covariance for both the vector field and the Hamiltonian at arbitrary states, with $H$ identified up to an additive constant fixed by an anchor $H(x_0)=H_0$; and
(iii) admits a finite-sample, high-probability error bound for the vector-field estimate that separates statistical fit from numerical discretization, exposing the roles of data points $\ell=K-M$ and multistep accuracy.
Kernel and PHS hyperparameters are selected by minimizing~\eqref{eq:nll}.

% From noisy, irregular trajectories $\mathcal D$ of the PHS \eqref{eq:PHS}–\eqref{eq:noisy_observations}, we aim to learn a probabilistic, PHS-consistent model of the drift $f(x)=[J(x)-R(x)]\nabla_x H(x)$ and the Hamiltonian $H(x)$ by embedding variable-step multistep (vLMM) constraints directly in a GP. 
% The method preserves energy balance/passivity, yields closed-form posteriors for $f$ and $H$ (with $H$ fixed up to an additive constant), and provides a finite-sample bound that separates statistical from discretization error; hyperparameters are chosen via marginal likelihood.

\subsection{Multistep Port-Hamiltonian System Kernel}
To perform exact GP inference of a Port–Hamiltonian vector field with irregularly-sampled trajectory data points, we combine the physics‐informed PHS prior from \cite{beckers2022gaussian} with the variable step size multistep integrator~ \cite{ensinger2024exact}.

\subsubsection{Multivariate Variable Step Size Multistep Projection}
With noisy state observations \(\{\tilde{{x}}_k\}_{k=1}^K\) at time steps \(\{t_k\}_{k=1}^K\), we define an order-$(p)$ vLMM~\cite[Ch.~3.5]{hairer1993solving} with the coefficient matrices, $A, B \in\mathbb{R}^{(K-M)\times K}$, parameterized by the step grid $\{h_k\}_{k=1}^{K-1}$, where $h_k = t_{k+1} - t_k$.
Let the control-affine system be:
\begin{equation}\label{eq:vlmm_dyn_constraints}
    \dot x(t) = f(x(t)) + g(x(t))u(t),
\end{equation}
with $f: \mathbb{R}^{n}\to \mathbb{R}^{n}$, ${g}: \mathbb{R}^{n}\to \mathbb{R}^{n\times m}$.
In the scalar case where $n=1$, dynamic constraints in \eqref{eq:vlmm_dyn_constraints} can be enforced through:
\begin{equation}\label{eq:AB_integrator}
    AX = B(\mathbf{f}(X) + \mathbf{g}(X)U),
\end{equation}
with the noiseless states $X = [x_1^\top, \dots, x_K^\top]^\top$ and inputs $U = [{u}_1^\top, \dots, {u}_K^\top]^\top$. 
The functions,
    $\mathbf{f}(X) = [f(x_1)^\top, \dots, f(x_K)^\top]^\top$ and
    $\mathbf{g}(X)U = [(g(x_1){u}_1)^\top, \dots, (g(x_K){u}_K)^\top]^\top,$
stacks the values of $f(x_k)$ and $g(x_k)u_k$, for $k=1 \dots K$, accordingly.
Notice that the forward Euler method with step sizes \(\{h_k\}\) for zero-input systems is a special case of \eqref{eq:AB_integrator}.
\begin{example}[Explicit Forward Euler]
    The discretization:
\[
  \dot x(t_k) = f(x(t_k))
  \quad\Longrightarrow\quad
  x_{k+1} - x_k = h_kf(x_k),
\]
can be written in the form:
\[
  AX = B\mathbf{f}(X),
  \quad
  \mathbf{f}(X)=\bigl[f(x_1),\dots,f(x_K)\bigr]^\top.
\]
Here \(A,B\in\mathbb{R}^{(K-1)\times K}\) have entries:
\begin{equation*}
  A_{k,i} =
    \begin{cases}
      -1, & i = k,\\
      +1, & i = k+1,\\
      0,  & \text{otherwise},
    \end{cases}
  \quad
  B_{k,i} =
    \begin{cases}
      h_k, & i = k,\\
      0,   & \text{otherwise},
    \end{cases}
\end{equation*}
so that: $(A X)_k = x_{k+1} - x_k$ and $(B \mathbf{f}(X))_k = h_k f(x_k)$.
\begin{comment}
    In matrix form, for variable step‐sizes \(\{h_k\}_{k=1}^{K-1}\):
    \scalebox{0.8}{$
      A = \begin{bmatrix}
        -1 & 1 & 0 & \cdots & 0\\
         0 & -1& 1 & \cdots & 0\\
         \vdots &   &\ddots&\ddots&\vdots\\
         0 & \cdots&0&-1&1
      \end{bmatrix},
      \quad
      B = \begin{bmatrix}
        h_1 & 0 & \cdots & 0 & 0\\
        0 & h_2 & \cdots & 0 &0\\
        \vdots & &\ddots &\vdots &\vdots\\
        0 & \cdots&0&h_{K-1} &0
      \end{bmatrix}.$
    }
\end{comment}
\end{example}
Multistep integrators of order greater than or equal to $1$ (e.g., Adams-Bashforth/Moulton, BDF etc.~\cite[Ch.~3]{hairer1993solving}) form updates from linear combinations of past states and vector-field evaluations. 
This class of numerical integrators yields the stacked linear constraint $A X = B \mathbf{f}(X)$ which preserves Gaussianity under a GP prior on $f$ and permits closed-form posterior updates~\cite{ensinger2024exact} and recursive GP extensions~\cite{huber2013recursive}.

In the multivariate case where $x \in \mathbb{R}^n$ for $n>1$, the dynamic constraint encoded in \eqref{eq:AB_integrator} can be conveniently generalized by augmenting the coefficient matrices:
\begin{align}
    A_{I} \coloneq A \otimes I_n,\ \  
    &B_{I} \coloneq B \otimes I_n,
\end{align}
where $A_{I}, B_{I}\in \mathbb{R}^{(K-M)n\times Kn}$.
The label vector $Y$ for GPR is constructed and related to the feature vector $X$ via:
\begin{equation}\label{eq:Y}
  Y \coloneq A_{I} \tilde X = B_{I} {(\mathbf{f}(X) + \mathbf{g}(X)U)} + \boldsymbol{\varepsilon}.
\end{equation}
\begin{remark}[Unmodeled GP Input Noise]\label{rm:unmodel_NIGP}
  In practice, the inputs $X$ appearing in $\mathbf{f}(X)$ and $\mathbf{g}(X)$ on the right-hand side of \eqref{eq:Y}
  are corrupted by measurement noise, so the GP is in fact heteroscedastic.  
  A common remedy is to perform a local Taylor expansion of $f(x+\epsilon)$ around each datum and absorb the induced input noise into an input‐dependent output‐noise term, e.g. noisy input Gaussian process (NIGP), \cite{mchutchon2011gaussian}.  
  While higher‐order expansions and full posterior corrections \cite{goldberg1997regression, kersting2007most} can further improve accuracy, they lie outside the scope of this work.  
  Accordingly, we assume noiseless inputs to the GP.  
\end{remark}
We now extend the exact GP inference framework proposed in \cite{ensinger2024exact} by injecting a physics-informed prior via the matrix-valued PHS kernel and construct the multistep port-Hamiltonian system (MS-PHS) kernel:
\begin{align}
  [K_Y]_{nm}
    &= \sum_{i=0}^{K}\sum_{j=0}^{K}
       [B_{I}]_{n,i} k_{\rm phs}({x}_i,{x}_j) [B_{I}]_{m,j},
    \label{eq:KY}\\
  [k_Y({x}_*)]_n
    &= \sum_{j=0}^{K} [B_{I}]_{n,j} k_{\rm phs}({x}_*,{{x}}_j),
    \label{eq:kY}
\end{align}
which is equivalent to the following matrix forms: 
\begin{equation}\label{eq:projected-kernels}
  K_Y = B_{I} K_{\rm phs} B_{I}^\top,
  \qquad
  k_Y({x}_*) = B_{I} k_{\rm phs}(X, {x}_*).
\end{equation}
The proposed MS-PHS kernel conditioning on \(Y\) via the standard GP posterior formulas then yields closed‐form posterior means and covariances over the continuous‐time PHS dynamics, exactly accounting for both the physics prior and the multistep integration constraints.

MS-PHS inherits passivity and closure under power-conserving interconnection from GP-PHS \cite[Prop.~1,2]{beckers2022gaussian}, a direct consequence of the closure of GPs under linear operators. 
Conditioning the GP prior on the multistep linear functionals yields a posterior GP with the same smoothness, so the gradients required for energy-rate and passivity calculations remain well-defined. 
Since the multistep projection acts componentwise and linearly on the port variables, it preserves energy additivity and the Port-Hamiltonian interconnection structure. 

\subsubsection{Exact Inference via Multistep PHS Kernel}
Given noisy observations $\mathcal{D}$, the posterior over the zero-input PHS vector field at a test input ${x}_*\in \mathbb{R}^n$ is:
\begin{align}
  {f}_* \mid \mathcal{D}, {x}_* &\sim \mathcal{N}\bigl(\mu_{{f}}, \Sigma_{{f}}\bigr),
\end{align}
with posterior predictive mean and covariance computed as
% {\footnotesize
\begin{align}
  \mu_{{f}}
    &= k_Y({x}_*)^\top\bigl(K_Y + \sigma_{{x}}^2 A_{I} A_{I}^\top\bigr)^{-1}(Y - B_{I}G(X)U),
    \label{eq:msphs-mean}\\
  \Sigma_{{f}}
    &= k_{\rm phs}(x_*,x_*) - k_Y({x}_*)^\top\bigl(K_Y + \sigma_{{x}}^2 A_{I} A_{I}^\top\bigr)^{-1}k_Y({x}_*),
    \label{eq:msphs-cov}
\end{align}
% }
where \(K_Y = K_Y^\top \succ 0 \in \mathbb{R}^{(K-M)n\times (K-M)n}\) is the training covariance from \eqref{eq:KY}, 
\(k_Y(x^*) \in \mathbb{R}^{(K-M)n}\) is the cross-covariance from \eqref{eq:kY},
$k_{\rm phs}(x^*,x^*)\in\mathbb{R}^{n\times n}$ is the raw PHS kernel evaluation at the test input,
and \(\sigma_{{x}}^2\) is the measurement‐noise variance defined in \eqref{eq:noisy_observations}.

To analyze the error composition of MS-PHS GP, we recall two building blocks of the spectral-flow learning framework~\cite{leung2025spectral}, vLMM theory and spectral regularization.
An order-\(p\) vLMM converts irregularly-sampled trajectory windows into linear constraints on stacked vector-field evaluations via a label map. 
Under the usual bounded step-ratio and order-\(p\) consistency assumptions~\cite{hairer1993solving}, the single-window local truncation error (LTE) for a vector field~\(f\), denoted as
\(\Delta_{\rm MS}(h_k;f)\), satisfies the pointwise bound:
\begin{equation}\label{eq:LTE}
    \|\Delta_{\rm MS}(h_k;f)\|_2 \le C_{\rm LTE}h_k^{p+1},
\end{equation}
with \(C_{\rm LTE}>0\) independent of the step \(h_k=t_{k+1}-t_k\). 

Spectral regularization~\cite{rosasco2005spectral} on the windowed flow Hilbert space $\mathcal W$ supplies the statistical building block. 
Let \(\mathsf T\) denote the population covariance operator on the windowed flow space and \({\mathsf g}_\lambda(\mathsf T)\) a spectral filter with Lipschitz exponent \(\mu\). 
Under a source condition of order \(r>0\) and with \(\ell\) labeled windows, standard perturbation and concentration arguments yield the finite-sample flow error rate, i.e., with probability at least \(1-\eta\), \cite[Thm.~1]{leung2025spectral}:
\begin{equation}\label{eq:flow_fs-hp}
    \|\widehat\Phi-\Phi_\rho\|_{\mathcal W}^{2} \leq C_{\rm flow}\log\frac{4}{\eta}\;\ell^{-\frac{2r}{2r+\beta}},\qquad
\beta:=\max\{1,2\mu\}.
\end{equation}
where $\ell>0$ is the number of labeled windows $\ell \coloneqq K-M$, $\widehat\Phi, \Phi_\rho\in \mathcal{W}$ are the estimated and target flow of \eqref{eq:PHS}.
Combining \eqref{eq:flow_fs-hp} with the vLMM LTE \eqref{eq:LTE} produces a two-term high-probability vector-field bound:
\begin{proposition}[Vector-field finite-sample high-probability bound]\label{pr:vec_field_pac}
    Under the assumption that $H$ lies in the base RKHS induced by $k_{\rm base}$,
    for $0 < \eta \leq 1$, the following inequality holds with probability at least $1-\eta$:
    \begin{equation}\label{eq:vf-pac-bound}
        \|\widehat f - f_\rho\|_{\mathcal{V}}^{2}
        % \mathcal{E}(\widehat f) - \mathcal{E}(f_\rho)
        \le
        \frac{C_{\rm fit}}{c_{\rm obs}^{2}(h)} \log\frac{4}{\eta} \ell^{-\frac{1}{5}}
         + 
        \frac{C_{\rm bias}}{c_{\rm obs}^{2}(h)}\ \mathbb E[h^{2p+2}],
    \end{equation}
    for some finite positive constant $c_{\rm obs}$, $C_{\rm fit}$, and $C_{\rm bias}$.
\end{proposition}
\begin{proof}
    Since Tikhonov regularization, e.g., GPR, has Lipschitz exponent \(\mu=2\), we take \(\beta=\max\{1,2\mu\}=4\).
    Under the assumption that $H$ lies in the base RKHS induced by $k_{\rm base}$, the baseline source condition \(r=\frac{1}{2}\), i.e., $\Phi_\rho\in\mathrm{Range}({\mathsf T}^{\frac{1}{2}})$ makes \cite[Thm.~2]{leung2025spectral} yield the rate exponent \(2r/(2r+\beta)=\frac{1}{5}\).
    If stronger smoothness justifies \(r=2\), the exponent improves to \(2r/(2r+\beta)=\frac{1}{2}\)~\cite[Sec.~3.2]{rosasco2005spectral}.
\end{proof}
The estimated and target vector field are denoted as $\widehat f, f_\rho \in \mathcal{V}$ respectively.
The $(\ell, h)$-independent constants $C_{\rm fit}, C_{\rm bias}$ are derived from $C_{\rm flow}$ and $C_{\rm LTE}$, respectively.
The inverse term $1/c_{\rm obs}^2(h)$ is a caveat that arises from inverting the discretization methods when inferring the underlying vector field, which scales up with respect to decreasing step sizes $h$.

Proposition~\ref{pr:vec_field_pac} shows that the estimation error splits into a statistical part that scales with $\ell$, and a discretization part that scales with $h$. 
Therefore, merely adding more data cannot remove discretization bias. Integrating the numerical integrator into the GP is necessary to obtain calibrated posteriors for vector fields and derived quantities like the Hamiltonian.

\subsection{Posterior Hamiltonian Surface}
We now show how to recover the full posterior over the Hamiltonian surface $H(x)$, given our noisy observations $\mathcal{D}$ and a noiseless anchor $H(0)=H_0$.
By placing a zero‐mean GP prior on $H$ with base kernel $k_{\rm base}(x,x')$ in \eqref{eq:Ham_prior} and combining it with the VM‐PHS kernel for the trajectory data, we obtain a joint Gaussian prior over the augmented vector:
\[
  {y}_{\rm aug} =
    \begin{bmatrix}
      H(0) \\
      {f}_{\rm vec}
    \end{bmatrix}
  \sim
  \mathcal{N} \bigl(0, K_{gg}\bigr),
\]
where ${f}_{\rm vec} \coloneq Y - B_{I}{\rm vec}(G(X)U)$ and the covariance matrix $K_{gg}$ is:
\begin{equation}\label{eq:Kgg}
  K_{gg} =
  \begin{bmatrix}
    K_{HH}(0,0)            & K_{Hf}(0, X) \\[6pt]
    K_{fH}(X, 0)           & K_{ff}(X,X)
  \end{bmatrix}.
\end{equation}
Each block is constructed as follows:
\begin{enumerate}
  \item Scalar block:
  \begin{equation*}
      K_{HH}(0,0)=k_{\rm base}(0,0)+\epsilon_H,
  \end{equation*}
  where $\epsilon_H$ is a small jitter ensuring positive definiteness.
  \item Multistep gradient block:
  \[
    K_{ff}(X,X)
      = K_Y + \sigma_{{x}}^2 A_{I} A_{I}^\top,
  \]
  which matches the training covariance in \eqref{eq:msphs-mean} and \eqref{eq:msphs-cov}.
  \item Cross‐covariance blocks:
    \begin{equation}\label{eq:K_Hf}
    K_{Hf}(0,X)
    =
    \begin{bmatrix}
        J_R(x_1)\nabla_{x_1}k_{\rm base}(0,x_1)\\
        \vdots\\
        J_R(x_K)\nabla_{x_K}k_{\rm base}(0,x_K)
    \end{bmatrix}B_{I}^\top,
    \end{equation}
  and $K_{fH}(X,0)=K_{Hf}(0,X)^\top$.
\end{enumerate}
Once the hyperparameters, e.g., ARD lengthscales~\cite{rasmussen2003gaussian, neal2012bayesian}, in $k_{\rm base}$ are learned through optimizing the negative log marginal likelihood~\eqref{eq:nll}, the surface posterior can be inferred by leveraging the linearity of differentiation operators \cite[Ch.~9.4]{rasmussen2003gaussian}.
Given the a priori anchor $H(0)=H_0$ and gradient information ${f}_{\rm vec}$ encoded as the data transformed by the multistep projection, 
the posterior at any test point ${x}_*$ is obtained by standard Gaussian conditioning:
\begin{equation}
    H_* \mid \mathcal{D}, H_0, {x}_* \sim \mathcal{N}(\mu_H, \sigma_H^2),
\end{equation}
where the posterior mean $\mu_H$ and variance $\sigma_H^2$ are:
\begin{align}
  \mu_H({x}_*)
    &= K_{*g}({x}_*)\bigl(K_{gg}\bigr)^{-1} {y}_{\rm aug},
    \label{eq:H-mean}\\
  \sigma_H^2({x}_*)
    &= k_{\rm base}({x}_*,{x}_*)
       - 
      K_{*g}({x}_*)\bigl(K_{gg}\bigr)^{-1}K_{*g}({x}_*)^\top,
    \label{eq:H-var}
\end{align}
where:
\begin{equation*}
    K_{*g}({x}_*)
    = \bigl[k_{\rm base}({x}_*,0),K_{Hf}({x}_*,X)\bigr]
    \quad\in\mathbb{R}^{1\times (1 + (K-M)n)}.
\end{equation*}
In particular, $K_{Hf}({x}_*,X)$ is constructed exactly as \eqref{eq:K_Hf} but replacing the anchor with the test input~${x}_*$.  
Equations \eqref{eq:H-mean}–\eqref{eq:H-var} then yield closed‐form expressions for the posterior mean and variance of the Hamiltonian surface.
\begin{remark}
    Since derivative information alone can recover the integrated surface only up to an arbitrary constant,
    anchoring $H(0)=H_0$ is essential to fix the additive constant of the GP. 
    Any other noiseless constraint on $H$ (e.g., $H({x}_i)=h_i$) can be incorporated analogously by augmenting $K_{gg}$ and ${y}_{\rm aug}$ with the corresponding rows and columns.
\end{remark}

\section{Experiments}\label{sc:experiment}
We evaluate key features of the MS-PHS GP framework, including vector field and Hamiltonian surface posterior recovery.
We learn the observation-noise variance and GP kernel hyperparameters by minimizing the negative log marginal likelihood~\eqref{eq:nll} with ARD length-scales~\cite{rasmussen2003gaussian, neal2012bayesian}, using Adam optimizer~\cite{kingma2014adam}.
All GP learning and quadratic programming modules are implemented using GPyTorch~\cite{gardner2018gpytorch} and qpth~\cite{amos2017optnet} with PyTorch~\cite{paszke2019pytorch}.

\subsection{Benchmarks}
Specifications of the three canonical oscillators as benchmarks for our posterior and recovery experiments: the Van der Pol oscillator, the Duffing oscillator, and a linear mass-spring-damper system are as follows:

\paragraph{Mass–Spring System}
As a linear baseline, we choose stiffness $k=1.0$, mass $m=1.0$, and damping $d=0.0$ in:
\[
  \ddot q = -\frac{k}{m} q -\frac{d}{m} \dot q,
\]
representing an undamped harmonic oscillator.

\paragraph{Van der Pol Oscillator} 
We set the nonlinearity parameter to $\mu = 1.0$ in the second‐order dynamics:
\[
  \ddot q = \mu (1 - q^2) \dot q - q,
\]
which exhibits relaxation oscillations and non‐conservative behavior.

\paragraph{Duffing Oscillator}
We use:
\[
  \ddot q = -\alpha q -\beta q^3 -\gamma \dot q,
\]
where $\alpha=1.0,\ \beta=5.0,\ \gamma=0.5$, yielding a double‐well potential with nonlinear stiffness and damping.

To generate training and test data, each system is integrated over $t\in[0,20]$ using a classical fourth‐order Runge–Kutta solver with fixed time step $\Delta t = 4\times10^{-3}$. 
The resulting trajectories are then corrupted by additive state noise drawn from $\mathcal{N}(0,\sigma_x^2 I)$ and irregularly subsampled at time points $t_k\in[0,20]$.
For the irregularly subsampling scheme, we start from an $N$-points uniform grid $\{\tau_k\}_{k=1}^N$ on $[0,20]$ with $N=100$ and add i.i.d. Gaussian “jitter”
$\epsilon_k \sim \mathcal N\left(0,\sigma_j^2\right)$ with $\sigma_{j}=0.05$.
The irregular timestamps are $t_k=\mathrm{clip}\left(\tau_k+\epsilon_k, t_0, t_1\right),$ after which they are sorted to enforce monotonicity.
This produces near-uniform sampling with small, zero-mean perturbations, allowing mild clustering/gaps without changing the average sampling rate.
Furthermore, we apply zero input ($u(t)\equiv0$) on the Van der Pol oscillator and a nominal sinusoidal $\cos(\omega t)$ on the Duffing oscillator and mass-spring system.
We compare three inference schemes:
\begin{enumerate}
  \item \textbf{MS-PHS + Adams-Bashforth (MS-PHS-ab):} our proposed multistep Port–Hamiltonian GP with Adams–Bashforth orders 1-3~\cite[Ch.~3]{hairer1993solving};
  \item \textbf{MS-ODE + Adams-Bashforth (MS-ODE-ab):} the same multistep GP exact inference without a PHS prior with Adams–Bashforth orders 1-3, similar to the method proposed in \cite{ensinger2024exact};
  \item \textbf{GP-PHS + LOESS/Savitzky-Golay (GP-PHS-loess/savgol):} a baseline that pre-processes the noisy data with a 
  locally estimated scatterplot smoothing (LOESS), a.k.a. locally weighted regression (LOWESS)~\cite{cleveland1979robust,cleveland1988locally} before fitting a PHS kernel \cite{beckers2022gaussian}.
  The Savitzky-Golay filter~\cite{schafer2011savitzky} can be seen as a regular grid variant of LOESS.
  We use locally quadratic fitting (loess-2) and Savitzky-Golay order 3 (savgol-3) in our simulations.
  % The Savitzky-Golay filter~\cite{schafer2011savitzky}
  % We use Savitzky-Golay order 3 (savgol-3) in this paper.
  % \item \textbf{GP-PHS + LOESS (GP-PHS-loess):} a baseline that pre-processes the noisy data with locally estimated scatterplot smoothing (LOESS), a.k.a. locally weighted regression (LOWESS)~\cite{cleveland1979robust,cleveland1988locally}, before fitting a PHS kernel~\cite{beckers2022gaussian}.
  % We use locally quadratic fitting (loess-2) in our simulations.
\end{enumerate}
The Adams-Bashforth scheme is one of the simplest multistep integrators that fits the form in \eqref{eq:AB_integrator}.
Alternatives include Adam-Moulton~\cite[Ch.~3]{hairer1993solving}, backward-difference formulas~\cite[Ch.~5]{wanner1996solving}.
An in-depth comparison of these schemes for ODE-based exact GP inference appears in ~\cite{ensinger2024exact}.

We use LOESS with GP-PHS as a strong baseline for two reasons.
First, the PHS kernel encodes a parametric Port-Hamiltonian structure with passivity/dissipation, injecting physics that provides a strong inductive bias for vector-field learning.
Second, LOESS is a nonparametric, non-causal local polynomial smoother that naturally handles irregular sampling through distance-based kernel weights.
Furthermore, LOESS is a distance-weighted local polynomial regression that works directly with irregular sampling grids.
Savitzky-Golay can be seen as an evenly-weighted local polynomial variant of LOESS that works best with regular sampling grids.

The drawback of this family of preprocess filters is their look-ahead: LOESS is not directly compatible with online/recursive GP inference and real-time control. 
Moreover, separating the filter from GPR inhibits proper propagation of observational uncertainty into the posterior covariance. 
Below, we show that MS-PHS offers the best of both worlds, a physics-informed prior that reduces directional misalignment and a GP kernel that embeds the numerical integrator, enabling calibrated Hamiltonian uncertainty for downstream controls.

\subsection{Vector Field Inference}
In this section, we assess how well each method recovers the drift $f(x)$ from noisy, irregularly-sampled trajectories. 
Fig.~\ref{fig:vf_mse_bars} shows the performance measured by the mean-squared error (MSE) of the predicted vector field on an evaluation mesh around the simulated true trajectory.
After $30$ runs per benchmark and method, we record the MSE and its 95\% confidence interval across varying observation noise, time-grid jitter, and GP ARD initializations.
In Fig.~\ref{fig:vf_mse_bars}, we notice that the proposed MS-PHS GP regression with AB-3 integrator performs at least comparable to, if not better than, GP-PHS and MS-ODE in the three benchmarks.
Furthermore, we observe that performance typically improves with increasing the order of the integrator.

\begin{figure}[t]%[h]
    \centering
    \includegraphics[width=\linewidth]{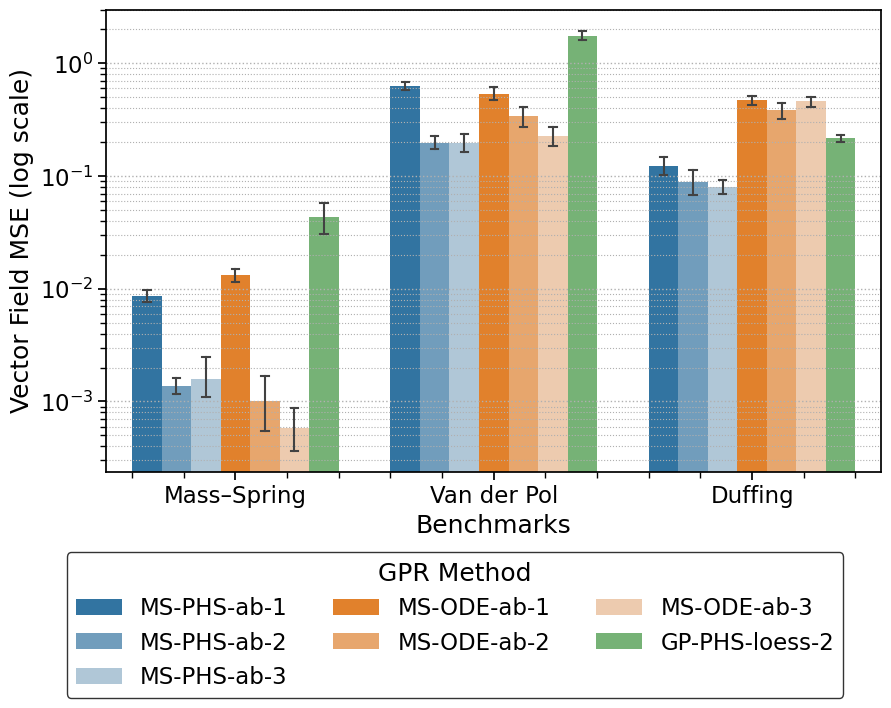}
    \caption{Vector-field mean squared error (MSE) on three dynamical benchmarks—Mass-Spring, Van der Pol, and Duffing oscillators. 
    Bars compare Gaussian-process regressors: multistep Port-Hamiltonian (MS-PHS-ab-1/2/3), multistep ODE (MS-ODE-ab-1/2/3), and a GP-PHS variant with Savitzky–Golay smoothing (GP-PHS-loess-2). 
    Error bars indicate variability across runs. Lower is better.}
    \label{fig:vf_mse_bars}
\end{figure}

To better understand the cause of error in Fig.~\ref{fig:vf_mse_bars}, we visualize the learned flows against the ground-truth field on Van der Pol, overlaying posterior uncertainty to localize directional misalignment and data-sparsity effects in Fig.~\ref{fig:vf_flow_comp}.
\begin{figure}[t]%[h]
    \centering
    \includegraphics[width=\linewidth]{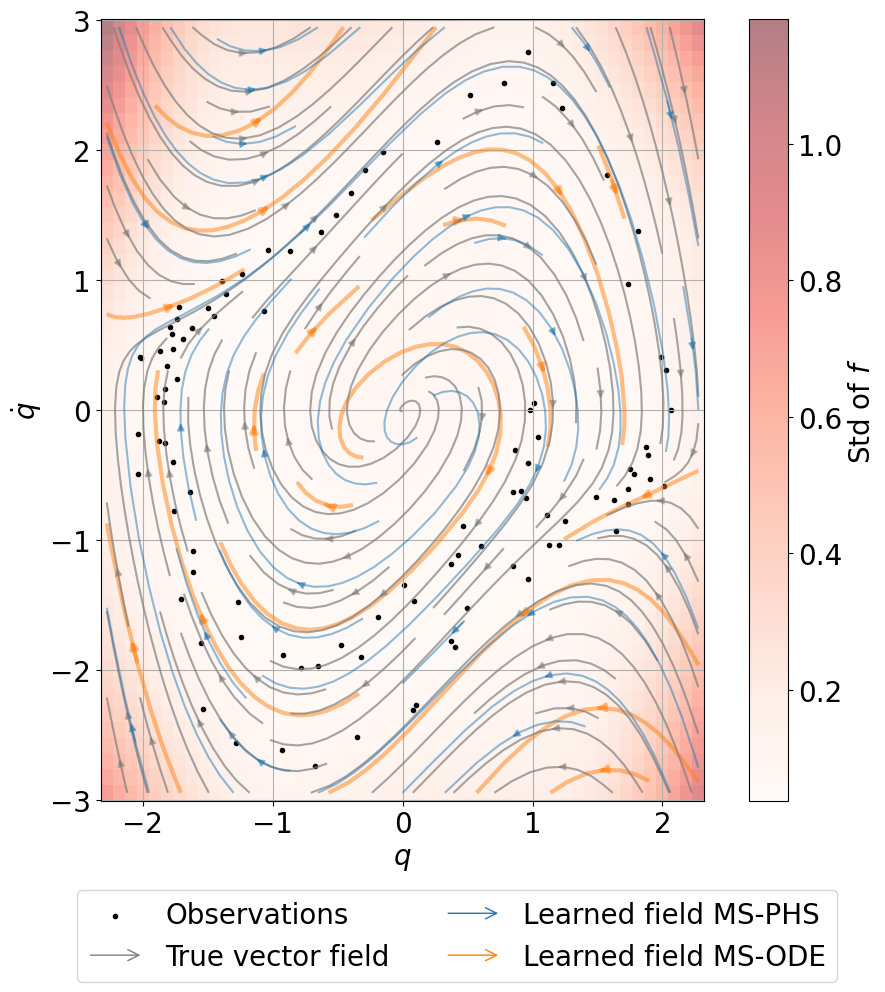}
    \caption{Van der Pol vector-field comparison. Ground truth (gray streamlines); MS-PHS GP (blue); MS-ODE GP (orange). 
    Black dots are observations; 
    red shading indicates high posterior standard deviation (std) of $f_*$, computed as $\sqrt{\sum_{i=1}^n(\Sigma_{{f}})_{ii}}$, with $\Sigma_{{f}}$ defined in \eqref{eq:msphs-cov}. 
    MS-PHS preserves the correct flow direction even in uncertain regions, whereas MS-ODE does not.
    }
    \label{fig:vf_flow_comp}
\end{figure}
The added PHS structure: skew-symmetric interconnection and dissipative terms tied to $\nabla H$ constrain $f$ to be approximately tangent to the energy level sets while inducing damping.
This physics prior regularizes extrapolation in data-sparse regions, reducing the spurious rotations seen with MS-ODE in Fig.~\ref{fig:vf_flow_comp}. 

To quantify directional agreement independent of scale, we next report the average cosine distance between true and learned vector fields with jittering grid effects turned off and replacing LOESS with Savitzky-Golay filter in GP-PHS.
In Table~\ref{tab:vf_cosine}, each entry is the mean cosine distance, with the standard deviation in parentheses, computed over the evaluation mesh across runs.
\begin{table}[t]%[h]
  \centering
  \caption{Learned Vector Field Average Cosine Distances. Entries are ``mean (standard deviation)'' across 30 runs.}
  \label{tab:vf_cosine}
  \begin{tabular}{l|c|c|c}
    \hline
    & \textbf{MS-PHS-ab-3} & \textbf{MS-ODE-ab-3} & \textbf{GP-PHS-savgol-3} \\ 
    \hline
    MassSpring & 0.001 (0.002) & 0.001 (0.005) & 0.003 (0.019) \\
    VanDerPol  & 0.003 (0.018) & 0.011 (0.034) & 0.002 (0.016) \\
    Duffing    & 0.003 (0.015) & 0.042 (0.093) & 0.003 (0.006) \\
    \hline
  \end{tabular}
\end{table}
Table~\ref{tab:vf_cosine} confirms the visual takeaway from Fig.~\ref{fig:vf_flow_comp}, that physics-informed methods yield smaller cosine distances and tighter variability across runs. 
A short summary of Table~\ref{tab:vf_cosine}:
\begin{enumerate}
    \item Mass–Spring: MS-PHS attains the best mean (tied with MS-ODE) but with a lower standard deviation, indicating more reliable alignment.
    \item Van der Pol: Both PHS methods match on mean and clearly outperform MS-ODE on both mean and standard deviation.
    \item Duffing: PHS methods are an order of magnitude better than MS-ODE, with low variance.
\end{enumerate}
Overall, the added PHS structure reduces directional misalignment and stabilizes extrapolation in sparse regions, as anticipated.

\subsection{Posterior Hamiltonian Recovery}
We next inspect the learned Hamiltonian geometry. Fig.~\ref{fig:duffing3D_ham_surf} overlays the MS-PHS posterior mean and its $\pm 2\sigma_H$ band on the ground-truth Duffing surface.
Uncertainty expands away from data (black points) and in high-curvature regions of the landscape, while remaining tight near visited states, consistent with the data support and the PHS prior.
\begin{figure}[t]
    \centering
    \includegraphics[width=\linewidth]{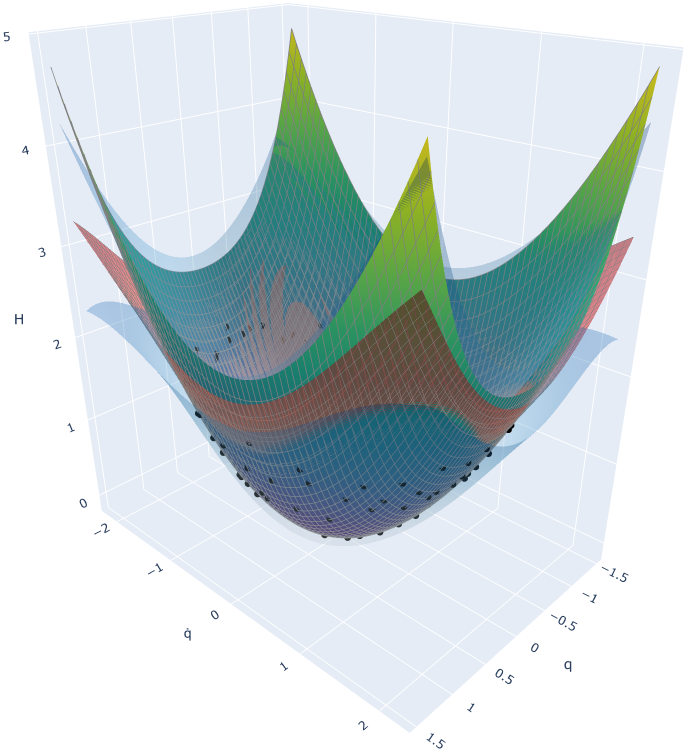}
    \caption{3D Hamiltonian landscape $H$ for the Duffing oscillator. The learned MS-PHS posterior mean $\mu_H$ surface (red) is overlaid with a reference/ground-truth surface $H_{\rm true}$ (green) and 95\% posterior band $\mu_H \pm 2\sigma_H$ (translucent blue) for comparison. 
    Black dots mark observed trajectory samples.}
    \label{fig:duffing3D_ham_surf}
\end{figure}

An important criterion for evaluating the Hamiltonian posterior is how well the learned Hamiltonian’s predictive variance tracks the true absolute error.
Fig.~\ref{fig:H_abserr_margin_comp} visualizes the true absolute error $|H_{\text{true}}-\mu_H|$ and the posterior standard deviation $\sigma_H$. 
Both quantities are small near the center of the explored region and grow toward the boundaries.
Importantly, the 95\% band $\mu_H\pm 2\sigma_H$ tracks the true error, indicating a well-calibrated posterior. 
We next compare the calibration of the Hamiltonian posterior variance for GP-PHS and MS-PHS.

\begin{figure}[t]
    \centering
    \includegraphics[width=\linewidth]{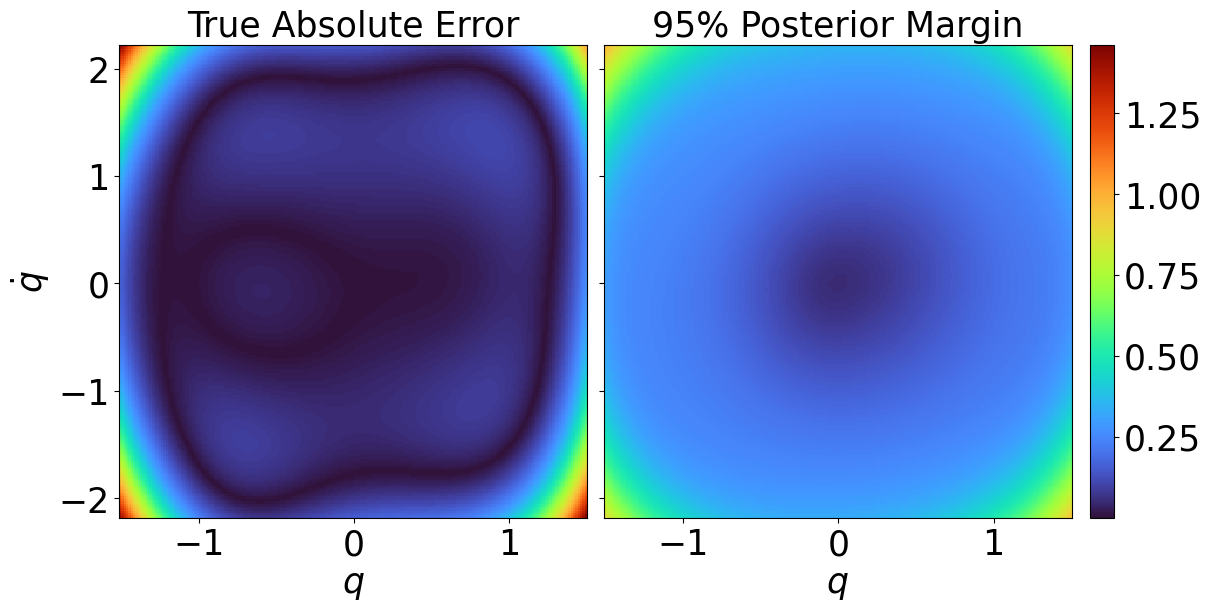}
    \caption{Duffing oscillator Hamiltonian error and uncertainty over the $(q,\dot q)$ plane. 
      (Left) true absolute error between ground truth and learned mean, $|H_{\text{true}}-\mu_H|$. 
      (Right) 95\% posterior margin, $1.96\sigma_H$. 
      The color bar shows magnitude; both error and uncertainty are small near the center and increase toward the domain boundary.
      Comparing the left panel and right panel shows that the $95\%$-posterior credible band ($\mu_H \pm 2\sigma_H$) tracks the true error: $|H_{\text{true}}-\mu_H|$.}
        \label{fig:H_abserr_margin_comp}
\end{figure}

\begin{figure}[t]
    \centering
    \includegraphics[width=\linewidth]{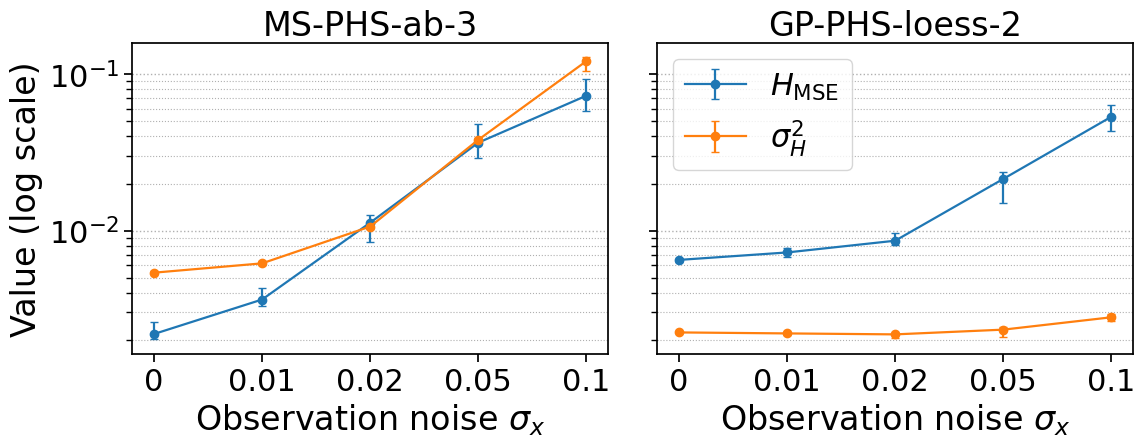}
    \caption{Scaling of Hamiltonian posterior error and uncertainty with observation noise on the Duffing system. 
    For each noise level $\sigma_x$, points show the median and bars the interquartile range (log scale). 
    Panels share the y-axis for direct comparison. 
    (Left) MS-PHS-ab-3: $\sigma_H^2$ rises in step with $H_{\text{mse}}$, indicating calibrated uncertainty across noise. 
    (Right) GP-PHS-loess-2: $H_{\text{mse}}$ increases sharply while $\sigma_H^2$ lags, especially for $\sigma_x^2\ge 0.02$, suggesting underestimation of uncertainty.
    }
    \label{fig:h_var_mse_comp_sibysi}
\end{figure}

\begin{figure}[t]
    \centering
    \includegraphics[width=\linewidth]{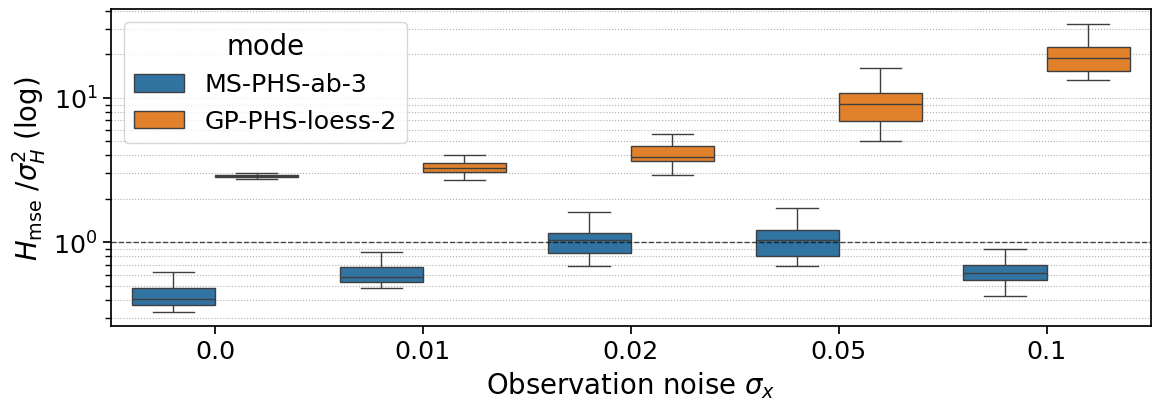}
    \caption{Error-uncertainty ratio of Hamiltonian uncertainty on Duffing. 
    Boxplots show the distribution of the $H_{\rm mse}/\sigma^2_H$ across observation noise levels $\sigma_x$, where dashed line represents ideal calibration at 1.
    MS-PHS-ab-3 remains near or below 1 across $\sigma_x$, indicating posterior variance that tracks the error, while GP-PHS-loess-2 rise above 1 as noise increases, evidence of overconfident posterior estimates.}
    \label{fig:h_var_mse_comp_ratio}
\end{figure}

In Fig.~\ref{fig:h_var_mse_comp_sibysi}, GP-PHS-savgol-3 shows error $H_{\mathrm{mse}}$ increasing rapidly with noise, whereas its reported uncertainty, summarized by $\sigma_H^2$, lags, especially for $\sigma_x^2 \ge 0.02$.
This mismatch is consistent with the decoupling between the LOESS smoothing stage and the GP kernel: the pre-processing can absorb and bias modeling errors that the GP posterior never sees, so the predictive variance fails to reflect the true model error.
By contrast, MS-PHS-ab-3 exhibits posterior variance that accurately tracks the realized error. 
In Fig.~\ref{fig:h_var_mse_comp_sibysi} (left), $\sigma_H^2$ rises in step with $H_{\rm{mse}}$ across noise levels, indicating that the posterior dispersion scales with the actual reconstruction difficulty.

To further investigate the issue of calibrating posterior variance to match the MSE of Hamiltonian, we leverage the notion of error-uncertainty ratio, defined as $H_{\rm mse}/\sigma^2_H$.
The posterior with $H_{\rm mse}/\sigma^2_H < 1$ is conservative, $H_{\rm mse}/\sigma^2_H> 1$ is overconfident, $H_{\rm mse}/\sigma^2_H = 1$ is well calibrated.
The calibration ratio in Fig.~\ref{fig:h_var_mse_comp_ratio} makes this explicit: the values concentrate near, or slightly below, the ideal line at $1$ for MS-PHS-ab-3; while GP-PHS-savgol-3 rise above $1$ as noise increases, i.e., uncertainty is underestimated.

\begin{table}[t]
\centering
\caption{Error-uncertainty ratio ($H_{\mathrm{mse}}/\sigma_H^2$) across time-grid jitter levels on Duffing. Entries are ``median [Q1, Q3]'' across 30 runs.}
\label{tab:duffing_jitter_ratio}
\begin{tabular}{lcc}
\toprule
$\sigma_j$ & MS-PHS-ab-3 & GP-PHS-loess-2 \\
\midrule
0.00 & 1.16 [0.847, 1.42] & 4.24 [3.28, 5.33] \\
0.01 & 1.07 [0.831, 1.29] & 10.1 [8.96, 11.4] \\
0.02 & 1.14 [0.929, 1.39] & 9.06 [7.54, 12.2] \\
0.05 & 1.04 [0.903, 1.36] & 9.52 [7.75, 10.5] \\
0.10 & 1.06 [0.927, 1.38] & 9.2 [7.22, 11] \\
\bottomrule
\end{tabular}
\end{table}
We observe similar phenomena under time-grid jitter. Table~\ref{tab:duffing_jitter_ratio} reports the ratio $\sigma_H^2 / H_{\mathrm{mse}}$ across jitter levels $\sigma_j^2$. 
MS-PHS-ab-3 stays close to unity, medians $\approx 1$ with relatively tight interquartile ranges, again signaling calibrated variance.
In contrast, GP-PHS-savgol-3 remains far below $1$ for all jitters, confirming that the separation between smoothing and inference leaves its predictive variance insensitive to the induced modeling error.

\section{Conclusion}
We introduced the MS-PHS GP kernel that enables learning of continuous-time dynamics from noisy, irregular trajectories while enforcing port-Hamiltonian structure. 
Placing a GP prior on $H$ and embedding variable-step multistep constraints as linear functionals yields closed-form posteriors for both the drift and the Hamiltonian. 
The construction preserves passivity and composability under power-conserving interconnections and decomposes error into statistical and discretization terms.

Empirically, MS-PHS recovers vector fields and Hamiltonian geometry more accurately and with better directional alignment than non-physics or prefiltered baselines.
Meanwhile, MS-PHS produces well-calibrated posterior uncertainty that tracks empirical error such as $H_{\rm mse}$ across noise and sampling jitter. 
These properties make MS-PHS especially attractive for downstream control tasks where physically consistent models and reliable uncertainty quantification are essential.

\bibliographystyle{IEEEtran}
\bibliography{refs}

\begin{thebibliography}{10}
\providecommand{\url}[1]{#1}
\csname url@samestyle\endcsname
\providecommand{\newblock}{\relax}
\providecommand{\bibinfo}[2]{#2}
\providecommand{\BIBentrySTDinterwordspacing}{\spaceskip=0pt\relax}
\providecommand{\BIBentryALTinterwordstretchfactor}{4}
\providecommand{\BIBentryALTinterwordspacing}{\spaceskip=\fontdimen2\font plus
\BIBentryALTinterwordstretchfactor\fontdimen3\font minus \fontdimen4\font\relax}
\providecommand{\BIBforeignlanguage}[2]{{%
\expandafter\ifx\csname l@#1\endcsname\relax
\typeout{** WARNING: IEEEtran.bst: No hyphenation pattern has been}%
\typeout{** loaded for the language `#1'. Using the pattern for}%
\typeout{** the default language instead.}%
\else
\language=\csname l@#1\endcsname
\fi
#2}}
\providecommand{\BIBdecl}{\relax}
\BIBdecl

\bibitem{rasmussen2003gaussian}
C.~E. Rasmussen, ``{G}aussian processes in machine learning,'' in \emph{Summer School on Machine Learning}.\hskip 1em plus 0.5em minus 0.4em\relax Springer, 2003, pp. 63--71.

\bibitem{alvarez2012kernels}
M.~A. Alvarez, L.~Rosasco, N.~D. Lawrence \emph{et~al.}, ``Kernels for vector-valued functions: A review,'' \emph{Foundations and Trends{\textregistered} in Machine Learning}, vol.~4, no.~3, pp. 195--266, 2012.

\bibitem{beckers2022gaussian}
T.~Beckers, J.~Seidman, P.~Perdikaris, and G.~J. Pappas, ``{G}aussian process port-{H}amiltonian systems: Bayesian learning with physics prior,'' in \emph{Proc. IEEE 61st Conference on Decision and Control (CDC)}.\hskip 1em plus 0.5em minus 0.4em\relax IEEE, 2022, pp. 1447--1453.

\bibitem{beckers2023data}
T.~Beckers, ``Data-driven bayesian control of port-{H}amiltonian systems,'' in \emph{Proc. 62nd IEEE Conference on Decision and Control (CDC)}.\hskip 1em plus 0.5em minus 0.4em\relax IEEE, 2023, pp. 8708--8713.

\bibitem{van2014port}
A.~Van Der~Schaft, D.~Jeltsema \emph{et~al.}, ``Port-{H}amiltonian systems theory: An introductory overview,'' \emph{Foundations and Trends{\textregistered} in Systems and Control}, vol.~1, no. 2-3, pp. 173--378, 2014.

\bibitem{leung2025spectral}
C.~H. Leung and P.~E. Par{\'e}, ``Spectral flow learning theory: Finite-sample guarantees for vector-field identification,'' \emph{arXiv preprint arXiv:2509.25000}, 2025.

\bibitem{ensinger2024exact}
K.~Ensinger, N.~Tagliapietra, S.~Ziesche, and S.~Trimpe, ``Exact inference for continuous-time {G}aussian process dynamics,'' in \emph{Proc. AAAI Conference on Artificial Intelligence}, vol.~38, 2024, pp. 11\,883--11\,891.

\bibitem{maschke1993port}
B.~M. Maschke and A.~J. van~der Schaft, ``Port-controlled {H}amiltonian systems: modelling origins and systemtheoretic properties,'' in \emph{Nonlinear Control Systems Design 1992}.\hskip 1em plus 0.5em minus 0.4em\relax Elsevier, 1993, pp. 359--365.

\bibitem{zaspel2024data}
P.~Zaspel and M.~G{\"u}nther, ``Data-driven identification of port-{H}amiltonian {DAE} systems by {G}aussian processes,'' \emph{arXiv preprint arXiv:2406.18726}, 2024.

\bibitem{neary2023compositional}
C.~Neary and U.~Topcu, ``Compositional learning of dynamical system models using port-{H}amiltonian neural networks,'' in \emph{Learning for Dynamics and Control Conference}.\hskip 1em plus 0.5em minus 0.4em\relax PMLR, 2023, pp. 679--691.

\bibitem{heinonen2018learning}
M.~Heinonen, C.~Yildiz, H.~Mannerstr{\"o}m, J.~Intosalmi, and H.~L{\"a}hdesm{\"a}ki, ``Learning unknown {ODE} models with {G}aussian processes,'' in \emph{International Conference on Machine Learning}.\hskip 1em plus 0.5em minus 0.4em\relax PMLR, 2018, pp. 1959--1968.

\bibitem{jagtap2020control}
P.~Jagtap, G.~J. Pappas, and M.~Zamani, ``Control barrier functions for unknown nonlinear systems using {G}aussian processes,'' in \emph{2020 59th IEEE Conference on Decision and Control (CDC)}.\hskip 1em plus 0.5em minus 0.4em\relax IEEE, 2020, pp. 3699--3704.

\bibitem{brunton2016discovering}
S.~L. Brunton, J.~L. Proctor, and J.~N. Kutz, ``Discovering governing equations from data by sparse identification of nonlinear dynamical systems,'' \emph{Proc. National Academy of Sciences}, vol. 113, no.~15, pp. 3932--3937, 2016.

\bibitem{hairer1993solving}
E.~Hairer, G.~Wanner, and S.~P. N{\o}rsett, \emph{Solving Ordinary Differential Equations I: Nonstiff Problems}.\hskip 1em plus 0.5em minus 0.4em\relax Springer, 1993.

\bibitem{atkinson2009numerical}
K.~Atkinson, W.~Han, and D.~E. Stewart, \emph{Numerical Solution of Ordinary Differential Equations}.\hskip 1em plus 0.5em minus 0.4em\relax John Wiley \& Sons, 2009.

\bibitem{huber2013recursive}
M.~F. Huber, ``Recursive {G}aussian process regression,'' in \emph{Proc. IEEE International Conference on Acoustics, Speech and Signal Processing}.\hskip 1em plus 0.5em minus 0.4em\relax IEEE, 2013, pp. 3362--3366.

\bibitem{mchutchon2011gaussian}
A.~McHutchon and C.~Rasmussen, ``{G}aussian process training with input noise,'' \emph{Advances in Neural Information Processing Systems}, vol.~24, 2011.

\bibitem{goldberg1997regression}
P.~Goldberg, C.~Williams, and C.~Bishop, ``Regression with input-dependent noise: A {G}aussian process treatment,'' \emph{Advances in Neural Information Processing Systems}, vol.~10, 1997.

\bibitem{kersting2007most}
K.~Kersting, C.~Plagemann, P.~Pfaff, and W.~Burgard, ``Most likely heteroscedastic {G}aussian process regression,'' in \emph{Proc. 24th International Conference on Machine Learning}, 2007, pp. 393--400.

\bibitem{rosasco2005spectral}
L.~Rosasco, E.~De~Vito, and A.~Verri, ``Spectral methods for regularization in learning theory,'' \emph{DISI, Universita degli Studi di Genova, Italy, Technical Report DISI-TR-05-18}, 2005.

\bibitem{neal2012bayesian}
R.~M. Neal, \emph{Bayesian Learning for Neural Networks}.\hskip 1em plus 0.5em minus 0.4em\relax Springer Science \& Business Media, 2012, vol. 118.

\bibitem{kingma2014adam}
D.~P. Kingma and J.~Ba, ``Adam: A method for stochastic optimization,'' \emph{arXiv preprint arXiv:1412.6980}, 2014.

\bibitem{gardner2018gpytorch}
J.~Gardner, G.~Pleiss, K.~Q. Weinberger, D.~Bindel, and A.~G. Wilson, ``Gpytorch: Blackbox matrix-matrix {G}aussian process inference with gpu acceleration,'' \emph{Advances in Neural Information Processing Systems}, vol.~31, 2018.

\bibitem{amos2017optnet}
B.~Amos and J.~Z. Kolter, ``{OptNet}: Differentiable optimization as a layer in neural networks,'' in \emph{Proc. International Conference on Machine learning}.\hskip 1em plus 0.5em minus 0.4em\relax PMLR, 2017, pp. 136--145.

\bibitem{paszke2019pytorch}
A.~Paszke, S.~Gross, F.~Massa, A.~Lerer, J.~Bradbury, G.~Chanan, T.~Killeen, Z.~Lin, N.~Gimelshein, L.~Antiga \emph{et~al.}, ``Pytorch: An imperative style, high-performance deep learning library,'' \emph{Advances in Neural Information Processing Systems}, vol.~32, 2019.

\bibitem{cleveland1979robust}
W.~S. Cleveland, ``Robust locally weighted regression and smoothing scatterplots,'' \emph{Journal of the American Statistical Association}, vol.~74, no. 368, pp. 829--836, 1979.

\bibitem{cleveland1988locally}
W.~S. Cleveland and S.~J. Devlin, ``Locally weighted regression: An approach to regression analysis by local fitting,'' \emph{Journal of the American Statistical Association}, vol.~83, no. 403, pp. 596--610, 1988.

\bibitem{schafer2011savitzky}
R.~W. Schafer, ``What is a {S}avitzky-{G}olay filter [lecture notes],'' \emph{IEEE Signal Processing Magazine}, vol.~28, no.~4, pp. 111--117, 2011.

\bibitem{wanner1996solving}
G.~Wanner and E.~Hairer, \emph{Solving Ordinary Differential Equations II}.\hskip 1em plus 0.5em minus 0.4em\relax Springer Berlin Heidelberg New York, 1996, vol. 375.

\end{thebibliography}

\end{document}